\pdfoutput=1
\documentclass{article}

\usepackage{graphicx} 
\usepackage{booktabs}
\usepackage{amsthm}
\usepackage{amsmath,amssymb,amsfonts}
\usepackage{mathrsfs}
\usepackage{tikz}
\usepackage[all,pdf]{xy}
\usetikzlibrary{positioning,calc,arrows.meta,shapes.geometric,fit}
\usepackage{dsfont}

\newcommand{\prop}{P}
\newcommand{\agents}{N}

\newcommand{\hc}{\hat{C}}
\newcommand{\rcb}{R^*} 
\newcommand{\cl}{\mathit{cl}}
\newcommand{\bisim}{\leftrightarrows}

\newcommand{\axiom}[1]{\ensuremath{#1}}
\newcommand{\sys}{\ensuremath{\mathrm{CB}_n}}
\newcommand{\sysTwo}{\ensuremath{\mathrm{CB}_2}}

\newcommand{\mc}[1]{\ensuremath{\mathit{#1}}} 
\newcommand{\lang}{\mathcal{L}_C}

\newcommand{\putaway}[1]{}

\newtheorem{definition}{Definition}
\newtheorem{lemma}{Lemma}
\newtheorem{theorem}{Theorem}

\title{Common Belief Revisited\footnote{This paper is identical to the version appearing in the Festscrift in honour of Andreas Herzig's 65th birthday, published by College Publications (2025). Many thanks to Hans van Ditmarsch for significant input on this work. I also thank the anonymous reviewers for helpful comments.}
}
\author{Thomas {\AA}gotnes\\University of Bergen and Shanxi University}
\date{}

\begin{document}

\maketitle

\begin{abstract}
    Contrary to common belief, common belief is not KD4. 
    If individual belief is KD45, common belief does indeed lose the 5 property and keep the D and 4 properties -- and it has none of the other commonly considered properties of knowledge and belief. But it has another property\footnote{This was pointed out to me by Hans van Ditmarsch who knew it from Andreas Herzig who knew it from Giacomo Bonanno.}: $C(C\phi \rightarrow \phi)$ -- corresponding to so-called shift-reflexivity (reflexivity one step ahead). This observation begs the question\footnote{First raised, to me at least, by Andreas Herzig. This is typical: Andreas is a genuinely curious researcher who has a gift for asking the right questions and is generous with sharing his ideas. No wonder he is one of the most influential researchers in the area of (multi-)agent logic in general and epistemic logic in particular. Andreas has conjectured (personal communication) that the answer to the question is ``yes''. He has also referred to it as (to him) ``the most important open problem in epistemic logic'' (personal communication, Hans van Ditmarsch). I am happy to be able to settle this problem in this paper, in honour of Andreas' 65th birthday. I know that he would be absolutely delighted if the answer is ``yes'' and the conjecture is correct. I also know that he would be even more delighted if the answer is ``no''.
    }:   
    is KD4 extended with this axiom a complete characterisation of common belief in the KD45 case? If not, what \emph{is} the logic of common belief?  In this paper we show that the answer to the first question is ``no'': there is one additional axiom, and, furthermore, it relies on the number of agents. We show that the result is a complete characterisation of common belief, settling the open problem.
\end{abstract}

\section{Introduction}

In standard (modal) logics of knowledge and belief \cite{Fagin:1995hc,Meyer1995,vDvdHK2007}, \emph{common} knowledge and belief \cite{lewis1969convention} are defined by taking the transitive closure of the union of the accessibility relations for the individual agents. It is well known that if the latter are equivalence relations, so is the former -- when individual knowledge has the S5 properties then so does common knowledge. What about weaker notions of belief? The most commonly used model of belief is KD45. 
It is also well known that common belief in that case ``inherits'' the D and 4 properties, but not the negative introspection property 5. 
Indeed, among the most commonly considered properties, D and 4 (in addition to the standard properties of normal modalities) are in a certain sense the only properties of common belief over KD45 \cite{aagotnes2021group}. 

That doesn't necessarily mean that common belief on KD45 \emph{is} KD4, and that there are not \emph{other} properties. And indeed there are. The formula
\begin{equation}   
\tag{\axiom{Cc}}
C (C\phi \rightarrow \phi)
\end{equation}
(where $C\phi$ means that $\phi$ is common belief by the grand coalition of all agents) is valid on KD45\footnote{This observation is attributed to Giacomo Bonanno (personal communication, Andreas Herzig).}.
To see this, observe that Euclidicity implies shift-reflexivity\footnote{A relation $R$ is shift-reflexive iff $Rxy$ implies $Ryy$ for all $x$ and $y$.}, and thus individual Eucludicity ensures that the common belief relation is reflexive in any state that is accessible by any agent from any other state.

This again begs the question: are there any other properties, or is KD4Cc a complete characterisation of common belief? 

Consider the case that there are only two agents, and a formula of the form:
\begin{equation}
    \label{eq:c3}
    \tag{$\hc 2$}
    (\hc\phi_1 \wedge \hc\phi_2 \wedge \hc \phi_3) \rightarrow \hc((\hc \phi_1 \wedge \hc\phi_2) \vee (\hc \phi_1 \wedge \hc\phi_3) \vee (\hc \phi_2 \wedge \hc\phi_3))
\end{equation}
where $\hc \phi = \neg C\neg \phi$. It is not too hard to see that this formula is valid: since there are only two agents, the first step on two of the paths to the three formulas must be for the same agent, and by individual Euclidicity there is a path to two of those formulas after one step.

This means, first, that KD4Cc is not a complete characterisation (it is not too hard to see that $\hc 2$ cannot be derived). Second, observe that $\hc 2$ is \emph{not} valid if there are three or more agents, so that means that a complete characterisation would be different for different numbers of agents.

In this paper we show that \emph{that's it}: KD4 + \axiom{Cc} + \axiom{\hc 2} is a sound and complete characterisation of common belief over KD45 for the language where the only modality is a common belief operator for the grand coalition, in the case of two agents -- and similarly for any number of agents. 

Existing axiomatisations of common (knowledge and) belief are for languages that also have individual belief modalities, and common belief is characterised in terms of individual belief by fixed-point axioms\footnote{See also \cite{herzig2020axiomatisation} for an intuitively elegant alternative axiom in terms of \emph{knowing-whether}, that works only for logics with the T axiom.} \cite{lehmann1984knowledge,bonanno1996logic,lismont1994logic} like (where $K_i$ is the individual knowledge modality for agent $i$ and $\agents$ is the set of all agents)
\[C(\phi \rightarrow \bigwedge_{i \in \agents}K_i\phi) \rightarrow (\bigwedge_{i \in \agents}K_i\phi \rightarrow C\phi)\]
(sometimes an induction \emph{rule} is used instead \cite{Fagin:1995hc}). While this, together with axioms describing the properties of individual belief, indirectly gives us a precise and complete characterisation of common belief, it obfuscates the properties of common belief since they are entangled with individual belief in the description. By leaving individual belief out of the picture on the syntactic level and having only a single modality, for common belief of the grand coalition, in the language, in this paper we get a direct and explicit complete characterisation of the core properties of common belief.

The rest of the paper is organised as follows. In the next section we formally define the language and the semantics, and the axiomatisations are presented and shown to be sound in Section \ref{sec:ax}. The main result, completeness, is shown in Section \ref{sec:completeness}. We briefly discuss taking the \emph{reflexive} transitive closure instead of the transitive closure in Section \ref{sec:reflexive}, and conclude in Section \ref{sec:conclusions}.

\section{Language and Semantics}

The language $\lang$ is defined as follows, given a set of atomic propositions $\prop$.
\[\phi ::= p \mid \neg \phi \mid \phi \wedge \phi \mid C \phi\]
where $p \in \prop$. We write $\hc \phi$ for $\neg C\neg \phi$, in addition to using the usual derived propositional connectives.

The language is interpreted in multi-agent Kripke models $M = (W,R,V)$ over $\prop$ and a finite set of agents $\agents$ (without loss of generality we assume that $N = \{1,\ldots, n\}$):
\begin{itemize}
    \item $W$ is a non-empty set of states;
    \item $R_i \subseteq W\times W$ is an accessibility relation for each agent $i \in \agents$;
    \item $V:\prop \rightarrow W$ is a valuation function.
\end{itemize}
A KD45 model is a model where each accessibility relation is serial, transitive and Euclidian. The class of all KD45 models for $n$ agents ($|N| = n$) is denoted \mc{KD45_n}.

Let
\[\rcb = (\bigcup_{i \in N} R_i)^*\]
where $Q^*$ is the transitive closure of the binary relation $Q$ on $W$.

The language is interpreted in these models as follows:

\[\begin{array}{lcl}
M,s \models p & \Leftrightarrow& s \in V(p)\\
M,s \models \neg\phi & \Leftrightarrow& M,s \not\models \phi\\
M,s \models \phi \wedge \psi & \Leftrightarrow& M,s \models \phi \text{ and } M,s \models \psi\\
M,s \models C \phi & \Leftrightarrow& \forall t \in W: \rcb st \Rightarrow M,t \models \phi
\end{array}\]

\section{Axiomatisation}
\label{sec:ax}

As discussed in the introduction, common belief in the KD45 case has the 4 and D properties but (as in most cases with individual negative introspection with the exception of S5 \cite{aagotnes2021group}) loses the 5 property. A simple example illustrating the latter is shown in Figure \ref{fig:simple}.

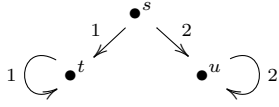
\begin{figure}
    \centering
\[\xymatrix@R=1em@C=1em{
&\bullet^s\ar@{->}[dl]_1\ar@{->}[dr]^2\\
\bullet^t\ar@(ul,dl)[]_{1}&& \bullet^u\ar@(ur,dr)[]^{2}
}\]
    \caption{Simple KD45 model.}
    \label{fig:simple}
\end{figure}

As mentioned in the introduction, we also get the additional property
\begin{equation}
    \tag{\axiom{Cc}} 
    C(C\phi\rightarrow \phi)
\end{equation}
This is perhaps easiest seen by observing that in a KD45 model the $R^*$ relation is shift-reflexive, since any state that is accessible from some other state by at least one agent has a reflexive loop for that agent due to individual Euclidicity. In particular, if a formula $\phi$ is satisfied in some state $s$ of a model then every state with the possible exception of the initial state $s$ in the generated submodel of that model from $s$ is reflexive. Of course, this property is not ``new'' -- as just argued it is implied by (is a sub-property of) Euclidicity. So, we don't lose 5 \emph{completely}, only \emph{partially}.

The second ``new'' class of properties we get are the following:
\begin{equation}
    \tag{\axiom{Cn}} 
C \left(\bigwedge_{1 \leq i < j \leq n+1} (C\phi_i \vee C\phi_j)\right) \rightarrow \bigvee_{1 \leq i \leq n+1} C\phi_i
\end{equation}
where $n$ is the number of agents. It might be instructive to observe that the following is equivalent:
\begin{equation}
    \tag{$\hc n$} 
\bigwedge_{1 \leq i \leq n+1} \hc \phi_i \rightarrow \hc \bigvee_{1 \leq i < j \leq n+1}(\hc \phi_i \wedge \hc \phi_j)
\end{equation}

\begin{lemma}
    \label{lemma:canonicity}
    \axiom{Cn} is canonical for the (FOL) property
    \[\forall x y_1\ldots y_{n+1} \left((\bigwedge_{1 \leq i \leq n+1} Rxy_i) \rightarrow \exists z \left(Rxz \wedge \bigvee_{1 \leq i < j \leq n+1} (Rzy_i \wedge Rzy_{j})\right)\right)\]
\end{lemma}
\begin{proof}
    \axiom{\hc n} is a (``very simple'') Sahlqvist formula, and it can be shown that the property in the lemma is its first-order correspondent (see \cite[Theorem 3.42]{Blackburn2001}). Thus, the formula is canonical for that property (see \cite[Theorem 4.42]{Blackburn2001}).
\end{proof}

In other words, this property says that if I can see $n+1$ states, then I can see a state that can see two of them. Again, this property is not ``new'' -- it is implied by the \axiom{5} axiom (Euclidicity, in which case all of those $n+1$ states can see each other).

Thus, instead of just dropping the 5 axiom we replace it with two weaker axioms: \axiom{Cn} and \axiom{Cc}. By adding those axioms to KD4, we get the axiomatisation $\sys$ shown in Table \ref{tab:ax}.

\begin{table}[t!]
    \centering
    \begin{tabular}{ll}\toprule
        all instances of propositional tautologies& \axiom{Prop}\\
        $C (\phi \rightarrow \psi) \rightarrow (C \phi \rightarrow C \psi)$& \axiom{K}\\
        $C\phi \rightarrow \neg C \neg \phi$& \axiom{D}\\
        $C\phi \rightarrow CC\phi$& \axiom{4}\\
        \midrule 
        $C(C\phi \rightarrow \phi$)& \axiom{Cc}\\
        $(C \bigwedge_{1 \leq i < j \leq n+1} (C\phi_i \vee C\phi_j)) \rightarrow \bigvee_{1 \leq i \leq n+1} C\phi_i$& \axiom{Cn}\\
        \midrule
        From $\phi \rightarrow \psi$ and $\phi$, derive $\psi$& \axiom{MP}\\
        From $\phi$, derive $C \phi$& \axiom{Nec}\\
    \bottomrule
    \end{tabular}
    \caption{Axiomatisation $\sys$ over the language $\lang$.}
    \label{tab:ax}
\end{table}

\begin{lemma}
    For any $n \geq 2$, $\sys$ is sound wrt. the class of \mc{KD45_n} models.
\end{lemma}
\begin{proof}
    Validity (preservation) for \axiom{K} and \axiom{Nec} follow from the fact that $C$ is normal (has standard relational semantics). Validity of \axiom{D} and \axiom{4} is well known \cite{aagotnes2021group}. 

    For \axiom{Cc}, for any state $w$ and agent $j$, by Euclidicity of $R_j$, if $R_jwv$ then $R_jvv$. Then also $R^*vv$. Thus, if $M,v \models C\phi$ then $M,v \models \phi$.

    For \axiom{Cn}, we use the equivalent \axiom{\hc n}. Let $t_1, \ldots, t_{n+1}$ be such that $(s,t_i) \in \rcb$ and $M,t_i \models \phi_i$, for each $1 \leq i \leq n+1$. In other words, for each $i$ there is an $\rcb$-path $t_i^0t_i^1\cdots t_i^{n_i}$, such that $t_i^0 = s$ and $t_i^{n_i} = t_i$. Since $t_i^0 \rcb t_i^1$ for each $1 \leq i \leq n+1$ and there are only $n$ agents, two of those first steps on those $n+1$ paths must be for the same agent $j$: $s R_j t_j^1$ and $s R_j t_k^1$ for some $j \neq k$. By Euclidicy of $R_j$ then we also have that $t_j^1 R_j t_k^1$. Then $M,t_j^1 \models \hc \phi_j \wedge \hc \phi_k$, and $M,s \models \hc(\hc \phi_j \wedge \hc \phi_k)$.
\end{proof}

Note that \axiom{Cn} does not hold for the case of $m > n$ agents. For example, \axiom{C2} holds in the case of two agents but not in the case of three. We thus get different systems $\sys$ for different numbers of agents.

These two additional axioms are all we need, as we now show.

\section{Completeness}
\label{sec:completeness}

We construct a satisfying \mc{KD45_n} model for any $\sys$-consistent formula. Two be able to focus on the key ideas we give the proof in full detail for the case $n=2$; in Section \ref{sec:n} we describe how it is generalised. Thus, henceforth assume that $n=2$.

The main ideas are as follows, with some pointers to the technical details that follow:
\begin{itemize}
    \item We take the standard canonical uni-modal model for $\sys$ (with a single relation interpreting the $C$ modality) as the starting point (Def. \ref{def:cm} below).
    \item We need a model where the transitions in the relation are ``labeled'' by agent names. However, we cannot just label all the transitions in the canonical model by some agent -- some of them correspond not to single agents but to sequences of agents due to transitivity of the common belief relation. 
    \item Also, if we label two outgoing transitions from the same state by the same agent, we need to make sure that the two incoming states are also related, due to individual Euclidicity.
    \item Key idea number one: if a state (1) is reflexive and (2) has at most one other incoming transition already labeled by an agent name, say agent 1, we can do the following: if the state has $k$\footnote{This also works when the number of outgoing transitions is not finite.} outgoing transitions, split it into $k$ copies with universal access between them for agent 1, each with one outgoing transition for agent 2. This does not affect satisfaction of any $\lang$ formulas; in fact the resulting model will be bisimilar when we take the transitive closure of union of the new accessibility relations (Lemma \ref{lemma:bisim} below).
    \item If we take the generated submodel of the canonical uni-modal model, then \emph{all reachable} states will be reflexive, taking care of condition (1).
    \item .. and we can also make it into a tree-like model in (pretty much) the standard way, taking care of condition (2).
    \item This transformation can be done recursively, for each state, possibly except the initial state which might not be reflexive, by \emph{alternating the agent names}: states with ingoing transitions for agent 2 get split into an agent 2 cluster.
    \item The transformation takes care of individual Euclidicity, since each new state only has one outgoing transition in addition to the reflexive loop. It also takes care of individual transitivity by alternating between agents.
    \item Key idea number two: this leaves us with the initial state, and this is where the \axiom{C2} axiom is needed. The initial state might have infinitely many directly accessible states, but to satisfy a (finite) formula we only need finitely many of them, at most corresponding to each subformula $\hc \psi$ we need to satisfy. For each triple of those states, the \axiom{C2} axiom ensures that the initial state can access some state that can access two of them (possibly not among the already identified states). That state then acts as a ``proxy'' for those two states since they can be reached by transtitive closure, and therefore we no longer need to be able to access those two states directly by the relation for some agent. We can thus replace those two states with the new one, and by repeating the process we can get down to at most two states that can access all the other needed states which thus can be reached through transitive closure. The transitions from the initial state to those two states can each be labeled with one of the two agents. (The set $X^\cl_w$ defined below, where $\cl$ gives us the set of $\hc \psi$ subformulas and $w$ is the initial state, is the set of those two states).
    \item The model construction proceeds recursively, building up a tree-like model level by level. Start with the initial node $w$ and its two successors identified above; make one copy of those successors for each outgoing transition (in the canonical model) each receiving an ingoing transition from $w$ for the same agent; add universal access for the same agent inside the cluster; repeat for the outgoing transitions in the new nodes but for the other agent. The construction is illustrated in Figure \ref{fig:construction}.
\end{itemize}

\begin{figure}
    \centering
\[\xymatrix{
Y_0:&&&&w
\ar@{->}!<0ex,0ex>;[dll]!<0ex,-1ex>
\ar@{->}!<0ex,0ex>;[dll]!<0ex,0ex>
\ar@{->}!<0ex,0ex>;[dll]!<0ex,1ex>_1
\ar@{->}!<0ex,0ex>;[drr]!<0ex,-1ex>
\ar@{->}!<0ex,0ex>;[drr]!<0ex,0ex>
\ar@{->}!<0ex,0ex>;[drr]!<0ex,1ex>
\ar@{->}!<0ex,0ex>;[drr]!<0ex,2ex>^2
\\
Y_1:&&\ldots
\ar@{->}!<0ex,0ex>;[dl]!<0ex,-1ex>
\ar@{->}!<0ex,0ex>;[dl]!<0ex,0ex>
\ar@{->}!<0ex,0ex>;[dl]!<0ex,1ex>_2
\ar@{->}!<0ex,0ex>;[dr]!<0ex,-1ex>
\ar@{->}!<0ex,0ex>;[dr]!<0ex,0ex>^2
&&&&\ldots
\ar@{->}!<0ex,0ex>;[d]!<-1ex,0ex>
\ar@{->}!<0ex,0ex>;[d]!<0ex,0ex>
\ar@{->}!<0ex,0ex>;[d]!<1ex,0ex>^1
\ar@{->}!<0ex,0ex>;[dl]!<-1ex,0ex>_1
\ar@{->}!<0ex,0ex>;[dl]!<0ex,0ex>
\ar@{->}!<0ex,0ex>;[dl]!<1ex,0ex>
\\
Y_2:&\ldots
\ar@{->}!<0ex,0ex>;[dr]!<0ex,-1ex>
\ar@{->}!<0ex,0ex>;[dr]!<0ex,0ex>
\ar@{->}!<0ex,0ex>;[dr]!<0ex,1ex>
\ar@{->}!<0ex,0ex>;[dr]!<0ex,2ex>^1
&&\ldots
\ar@{->}!<0ex,0ex>;[d]!<-1ex,0ex>
\ar@{->}!<0ex,0ex>;[d]!<0ex,0ex>
\ar@{->}!<0ex,0ex>;[d]!<1ex,0ex>^1
&&\ldots
\ar@{->}!<0ex,0ex>;[d]!<-1ex,0ex>
\ar@{->}!<0ex,0ex>;[d]!<0ex,0ex>
\ar@{->}!<0ex,0ex>;[d]!<1ex,0ex>^2
&\ldots\\
Y_3:&&\ldots&\ldots&&\ldots\\
\vdots
}\]
    \caption{Part of the model construction. Each cluster (denoted $\cdots$) has incoming transitions of only one type, has internal universal access for the same type, and has one outgoing transition of the opposite type for each node in the cluster.}
    \label{fig:construction}
\end{figure}
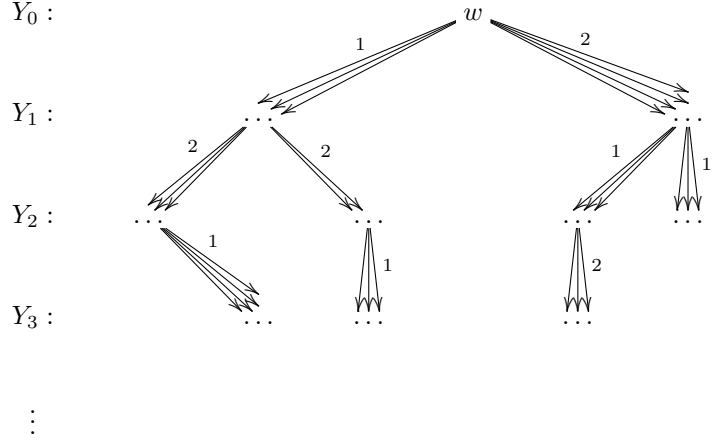

We proceed with the details. It might be helpful to keep an eye on Figure \ref{fig:construction}.

A \emph{uni-modal} model is a structure $M = (W,R,V)$ where $W$ and $V$ are like in a (multi-agent) model, and $R \subseteq W\times W$ is a single accessibility relation. The $\lang$ language is interpreted in uni-modal models by letting $M,s \models C\phi$ iff $M,t \models \phi$ for all $t \in W$ such that $Rst$ (and letting the other clauses be as in the interpretation in a model).

\begin{definition}[Canonical uni-modal model]
    \label{def:cm}
    The canonical uni-modal model $M^c = (W^c,R^c,V^c)$ is defined as follows:
    \begin{itemize}
        \item $W^c$ is the set of all maximal $\sys$-consistent sets;
        \item $R^cwv$ iff for all formulas $\psi$, if $\psi \in v$ then $\hc \psi \in w$;
        \item $V(p) = \{w: p \in W\}$.
    \end{itemize}
\end{definition}

We say that an MCS $\Gamma$ is \emph{branching} if $\Gamma R^c\Delta$ and $\Gamma R^c\Delta'$ for some MCSs $\Delta \neq \Delta'$. We say that a set of formulas $\cl$ is \emph{proper} iff it is finite, is closed under subformulas, contains $\hc \top$, and contains $\hc \neg \psi$ whenever it contains $C\psi$.

Let $w$ be a branching MCS and $\cl$ be proper set of formulas. We now define a finite set $X^\cl_w \subseteq W^c$ of states. We start with recursively defining $X^{(\cl,w)}_i$ for each natural number $i$. 

\begin{description}
    \item[$X^{(\cl,w)}_0$] Let $\psi_0, \ldots, \psi_k$ be the (finitely many) different formulas of the form $\hc \chi$ in $\cl$. For each $\psi_i = \hc\chi_i$, if $M^c,w \models \hc \chi_i$ let $u_i \in W^c$ be such that $R^cwu_i$ and $M^c,u_i \models \chi_i$. Finally let $X^{(\cl,w)}_0 = \{u_0, \ldots, u_k\}$.
    \item[$X^{(\cl,w)}_{i+1}$] If $X^{(\cl,w)}_i$ contains less than three nodes we are done, and let $X^{(\cl,w)}_{i+1} = X^{(\cl,w)}_i$. Otherwise, let $t,u,v$ be three (pair-wise) different nodes in $X^{(\cl,w)}_i$. By Lemma \ref{lemma:canonicity} (with $x = w$) there is a $z \in W^c$ such that $R^cwz$ that can see two of those three states. Without loss of generality assume that those are $t$ and $u$, i.e., that $R^czt$ and $R^czu$. Let $X^{(\cl,w)}_{i+1} = (X^{(\cl,w)}_i \setminus \{t,u\}) \cup \{z\}$.
\end{description}

\begin{lemma}
    For some $i\geq 0$, $X^{(\cl,w)}_{i+1} = X^{(\cl,w)}_i$.
\end{lemma}
\begin{proof}
    $X^{(\cl,w)}_0$ is finite by definition. $X^{(\cl,w)}_{i+1} = X^{(\cl,w)}_i \setminus \{t,u\} \cup \{z\}$ has at least one state less than $X^{(\cl,w)}_i$ (two if $z$ is already in $X^{(\cl,w)}_i$). 
\end{proof}
Thus, let $X^{(\cl,w)} = X^{(\cl,w)}_i$, where $i$ is the lowest number such that $X^{(\cl,w)}_{i+1} =  X^{(\cl,w)}_i$.

\begin{lemma}
    \label{lemma:l1}
    ~
    \begin{enumerate}
        \item $|X^{(\cl,w)}| > 0$
        \item $|X^{(\cl,w)}| < 3$
        \item for any $t \in X^{(\cl,w)}_0$ there is a $u \in X^{(\cl,w)}$ such that $R^cut$
    \end{enumerate}
\end{lemma}
\begin{proof}
    ~
    \begin{enumerate}
        \item By the \axiom{D} axiom, $\hc \top \in w$, so there is at least one state $u$ such that $R^c wu$. Each step in the elimination process removes at most two states when there is at least three states left.
        \item By definition, $X^{(\cl,w)}_{i+1}$ contains at least one state less than $X^{(\cl,w)}_i$ when the latter has more than two states.
        \item Let $t \in X^{(\cl,w)}_0$. We show that for any $i$, there is a $u \in X^{(\cl,w)}_i$ such that $R^cut$. For the base case let $u = t$: $R^ctt$ due to $R^cwt$ and the \axiom{Cc} axiom. For the inductive case, let $X^{(\cl,w)}_{i+1} = (X^{(\cl,w)}_i \setminus \{t',u'\}) \cup \{z\}$ where $R^cwz$ and $R^czt'$ and $R^czu'$. By the inductive hypothesis there is a $u \in X^{(\cl,w)}_i$ such that $R^cut$. If $u \neq t'$ and $u \neq u'$ then also $u \in X^{(\cl,w)}_{i+1}$. Consider the case that $u = t'$. Since $R^czt'$ and $R^cut$, we have that $R^czt$ by transitivity (axiom \axiom{4}). Thus, let $u = z$.
    \end{enumerate}
\end{proof}

Finally, we define the set $X^\cl_w$ (given $\cl$ and the branching MCS $w$). If $|X^{(\cl,w)}| = 2$, let $X^\cl_w = X^{(\cl,w)}$. Otherwise, let $u \in W^c$ be such that $R^cwu$ and $u \not \in X^{(\cl,w)}$ (it exists since $w$ is branching), and let $X^\cl_w = X^{(\cl,w)} \cup \{u\}$. In either case $|X^\cl_w| = 2$, so henceforth assume that $X^\cl_w = \{x_0,y_0\}$.

Given a proper $\cl$ and a branching MCS $w$, we are now going to build a model $M_{(\cl,w)} = (W,R,V)$.
We first define a set of states $Y_i$ and a set of paths $\Pi_i$, for each natural number $i$, by mutual recursion (keeping in mind that these sets are parameterised by $\cl$ and $w$). The former is in turn defined in terms of sets $Y^\pi_i$ where $\pi \in \Pi_i$ is a path. The empty path is denoted $\epsilon$ (as usual we write $\pi x$ for the concatination of path $\pi$ with state $x$, and when $\pi = \epsilon$ we omit $\pi$ in the concatination and write just $x$ for $\epsilon x$). Any non-empty path\footnote{We slightly abuse the word ``path'' here, as a sequence of symbols representing states. In fact, a path is actually a path in the resulting model, with the exception of the possible initial states $w^1$ or $w^2$ in the path which both represent the root node $w$. We use $w^1$ and $w^2$ to distinguish between left paths and right paths.} starts with either $w^1$ or $w^2$. A path that starts with $w^1$ is called a \emph{left path}; a path that starts with $w^2$ is called a \emph{right path}. The model construction is illustrated in Figure \ref{fig:construction}.
\begin{itemize}
    \item $i = 0$: $Y_0 = Y^\epsilon_0 = \{w\}$. $\Pi_0 = \{\epsilon\}$.
    \item $i = 1$ (recall that $X^\cl_w = \{x_0,y_0\})$:
    \begin{itemize}
        \item $\Pi_1 = \{w^1,w^2\}$
        \item $Y^{w^1}_1 = \{x^{w^1}_z : R^cx_0z\}$ and $Y^{w^2}_1 = \{y^{w^2}_z : R^cy_0z\}$
        \item $Y_1 = Y^{w1}_1 \cup Y^{w2}_1$
    \end{itemize}
    \item $i \geq 2$, $i = j+1$:
    \begin{itemize}
        \item $\Pi_i = \{\pi v : \pi \in \Pi_j, v \in Y_j^\pi\}$
        \item For each $\pi \in \Pi_j$ and $v = t^\pi_s \in Y^\pi_j$, $Y^{\pi v}_i = \{s^{\pi v}_z : R^csz\}$
        \item $Y_i = \bigcup_{\pi \in \Pi_i} Y^\pi_i$
    \end{itemize}
\end{itemize}
Finally, let the model $M_{(\cl,w)} = (W,R,V)$ be defined as follows: 
\begin{itemize}
    \item $W = \bigcup_{i \geq 0} Y_i$
    \item $R_1 uv$ iff $v\in Y_i^{\pi}$ for $i \geq 1$ and
    \begin{itemize}
        \item if $\pi$ is a left path:
            \begin{itemize}
        \item $i = 1$ and $u = w$, or
        \item $i\neq 1$ is odd and $u \in Y^{\pi'}_{i-1}$ such that $\pi = \pi'u$, or
        \item $i$ is odd and $u \in Y^\pi_i$.
    \end{itemize}
    \item if $\pi$ is a right path:
     \begin{itemize}
        \item $i$ is even and $u \in Y^{\pi'}_{i-1}$ such that $\pi = \pi'u$, or
        \item $i$ is even and $v \in Y^\pi_i$.
    \end{itemize}
    \end{itemize}
    \item $R_2 uv$ iff $v \in Y_i^\pi$ for $i \geq 1$ and (just swap ``left'' and ``right'' above):
        \begin{itemize}
        \item if $\pi$ is a right path:
            \begin{itemize}
        \item $i = 1$ and $u = w$, or
        \item $i\neq 1$ is odd and $u \in Y^{\pi'}_{i-1}$ such that $\pi = \pi'u$, or
        \item $i$ is odd and $u \in Y^\pi_i$.
    \end{itemize}
    \item if $\pi$ is a left path:
     \begin{itemize}
        \item $i$ is even and $u \in Y^{\pi'}_{i-1}$ such that $\pi = \pi'u$, or
        \item $i$ is even and $v \in Y^\pi_i$.
    \end{itemize}
    \end{itemize}
    \item $V(p) = \{s^\pi_z \in W: s \in V^c(p)\} \cup (\{w\}\cap V^c(p))$
\end{itemize}

\begin{lemma}
    \label{lemma:kd45}
    For any branching $w \in W^c$ and proper $\cl$, $M_{(\cl,w)}$ is a \mc{KD45_2} model.
\end{lemma}
\begin{proof}
For \emph{seriality}, consider the case for $R_1$. For $w$, we have that $R^1wx^{w^1}_{z'}$ for some $z'$ such that $R^cx_0z$ (exists since $R^c$ is serial). Let $x' = s^\pi_z \in Y^\pi_i$. First assume that $\pi$ is a left path. If $i$ is odd then $R_1x'x'$. If $i\neq 0$ is even, let $z'$ be such that $R^czz'$ (it exists since $R^c$ is serial). We have that $z^{\pi s^\pi_z}_{z'} \in Y_{i+1}^{\pi s^\pi_z}$, and $R_1s^\pi_z z^{\pi s^\pi_z}_{z'}$ by definition. The cases for the right path, and for $R_2$, are similar.

For \emph{transitivity}, consider frist $R_1$. Assume that $R_1 ab$ and $R_1 bc$, where $b \in Y^\pi_i$ for some $\pi$ and $i$. If $\pi$ is a left path, $R_1ab$ can only hold if $i$ is odd. Then $R_1bc$ implies that also $c \in Y^\pi_i$ since the only outgoing transitions from $Y^\pi_i$ when $i$ is odd and $\pi$ is left is to other nodes inside $Y^\pi_i$. Then, by definition, also $R_1ac$. If $\pi$ is a right path the argument is symmetric: then $R_1ab$ can only hold if $i$ is even. Then $R_1bc$ implies that also $c \in Y^\pi_i$ since the only outgoing transitions from $Y^\pi_i$ when $i$ is even and $\pi$ is right is to other nodes inside $Y^\pi_i$. Then, by definition, also $R_1ac$. Thus, $R_1$ is transitive. Transitivity of $R_2$ can be shown in exactly the same way.        

For \emph{euclidicity}, by construction, all outgoing transitions for $R_k$ ($k \in \{1,2\}$) from any state goes to the same set $Y^\pi_{i}$, and $R_k$ has universal accessibility inside that set. 
\end{proof}

Let $M^*_{(\cl,w)} = (W,R^*,V)$ be the uni-modal variant of $M_{(\cl,w)}$, which has the same state space and valuation function and where $R^* = (R_1 \cup R_2)^*$. The following is immediate from the definition:

\begin{lemma}
    \label{lemma:uni}
    For any formula $\phi$ and state $v \in W$, $M_{(\cl,w)},v \models \phi$ iff $M^*_{(\cl,w)},v \models \phi$
\end{lemma}

In the constructed model we have several copies $s^\pi_z$ of states $s$ in the canonical model, for each path $\pi$ and for each outgoing transition from $s$ to state $z$ in the canonical model. We show that every state in (the uni-modal) model $M^*_{(\cl,w)}$, except the inital state $w$, is bisimilar (see the appendix for definitions) to the corresponding state in the (uni-modal) canonical model.

\begin{lemma}
    \label{lemma:bisim}
    For any $s^\pi_z \in W$, $M^*_{(\cl,w)}, s^\pi_z \bisim M^c,s$.
\end{lemma}
\begin{proof}
    Let the relation $Z \subseteq (W\setminus \{w\}) \times W^c$ be defined by: $s^\pi_z Z s$. We show that it is a bisimulation:
    \begin{itemize}
        \item Atoms: immediate.
        \item Forth: Let $s^\pi_z Z s$ and $s^\pi_zR^*s'^{\pi'}_{z'}$. If $s' = s$, then $sR^cs'$ (since $s \neq w$). Assume that $s' \neq s$. The path from $s^\pi_z$ to $s'^{\pi'}_{z'}$ in $M_{(\cl,w)}$ consists of a sequence of steps some of which are inside the same cluster $Y_i^\pi$ (where all nodes correspond to the same node in $M^c$) and some that go from one cluster $Y_i^\pi{''}$ to the next $Y_{i+1}^{\pi{''} t^{\pi''}_x}$ such that $R^ctx$. Simply disregard the former, and we are left with a path from $s$ to $s'$ in $M^c$. Thus, $R^css'$, and we have that $s'^{\pi'}_{z'}Zs'$.
        \item Back: Let $s^\pi_z Z s$ and $R^css'$. Assume that $s^\pi_z \in Y^\pi_i$. Then also $s^\pi_{s'} \in Y^\pi_i$, and $s^\pi_zR_k s^\pi_{s'}$ for $k=1$ or $k=2$. By construction, $s^\pi_{s'}R_{\overline{k}} {s'}^{\pi s^\pi_{s'}}_{z'}$ for some $z'$, where $\overline{1} = 2$ and $\overline{2} = 1$. Thus, $s^\pi_zR^* s'^{\pi'}_{z'}$.
    \end{itemize}
\end{proof}

\begin{lemma}
    \label{lemma:l2}
    For any branching $w \in W^c$ and any proper $\cl$, for any formula $\phi \in \cl$, $M_{(\cl,w)},w \models \phi$ iff $M^c,w \models \phi$.
\end{lemma}
\begin{proof}
    The proof is by induction on the structure of $\phi \in \cl$. 
    \begin{itemize}
        \item $\phi = p \in \prop$: $M_{(\cl,w)},w \models p$ iff $w \in V(p)$ iff $w \in V^c(p)$ iff $M^c, w \models p$.
        \item Boolean connectives: straightforward.
        \item $\phi = C\psi$: for the implication towards the left, we show the contrapositive. Assume that $M_{(\cl,w)},w \not\models \phi$. $M_{(\cl,w)},w \models \hc \neg \psi$, i.e., there is a $s^\pi_z \in W$ such that $R^*ws^\pi_z$ and $M_{(\cl,w)},s^\pi_z \models \neg\psi$. The path from $w$ to $s^\pi_z$ in $M_{(\cl,w)}$ consists of a sequence of steps some of which are inside the same cluster $Y^\pi_i$ (where all nodes correspond to the same node in $M^c$) and some that go from one cluster $Y_i^\pi{'}$ to the next $Y_{i+1}^{\pi{'} t^{\pi{'}}_z}$ such that $R^ctz$. Simply disregard the former, and we are left with a path from $w$ to $s$ in $M^c$. By Lemmas \ref{lemma:uni} and \ref{lemma:bisim}, $M^c,s \models \neg \psi$. Thus, $M^c,w \models \hc \neg \psi$.

        For the direction towards the right, we show the contrapositive. Assume that $M^c,w \models \hc \neg \psi$. Since $\hc \neg\psi \in \cl$, there is a $v \in X_0^{(\cl,w)}$ such that $M^c, v \models \neg \psi$. By Lemma \ref{lemma:l1}.3, there is a $u \in X^{(\cl,w)}$ such that $R^cuv$. By construction of $M_{(\cl,w)}$ that means that $R_ku^{w^x}_v v^{w^x u^{w^x}_v}_z$ (for some $x \in \{1,2\}$ and $z$) for some $k \in \{1,2\}$, and since we also have that $R_jwu^{w^x}_v$ (for some $j$), we get that $R^*wv^{w^x u^{w^x}_v}_z$. By Lemma \ref{lemma:bisim}, $M_{(\cl,w)}^*,v^{w^x u^{w^x}_v}_z \models \neg \psi$ and by Lemma \ref{lemma:uni} $M_{(\cl,w)},v^{w^x u^{w^x}_v}_z \models \neg \psi$. Thus, $M_{(\cl,w)},w \models \hc \neg \psi$.
    \end{itemize}
\end{proof}

\begin{theorem}
    $\sysTwo$ is complete wrt. the class of all \mc{KD45_2} models.
\end{theorem}
\begin{proof}
    Assume that $\phi$ is consistent. Let $\phi' = \phi \wedge \hc p \wedge \hc \neg p$, for some $p$ not occurring in $\phi$. It is easy to see that $\phi'$ is consistent. By the standard Lindenbaum construction $\phi$ is included in some branching MCS $\Phi$. We have that $M^c,\Phi \models \phi$ by the standard truth lemma for canonical (in this case uni-modal) models. Let $\cl$ be the smallest set containing all subformulas of $\phi'$ as well as $\hc \top$ and that is closed under the rule: if $C\psi \in \cl$ then $\hc \neg\psi \in \cl$ (this set is finite). By Lemma \ref{lemma:l2} $M_{(\cl,\Phi)},\Phi \models \phi$. Thus, $\phi$ is satisfiable on a \mc{KD45_2} model (Lemma \ref{lemma:kd45}).
\end{proof}

\subsection{The general case}
\label{sec:n}

The proof for the case that $n > 2$ is almost identical, with only two minor differences:

\begin{itemize}
    \item $n$ outgoing transitions from the initial state: in the first step of the model construction, we use the \axiom{Cn} axiom in exactly the same way to get the set $X^\cl_w$, now consisting of $n$ (instead of 2) states.
    \item Choice of alternating agents and seriality: we only need two agents for alternation, so for each of those $n$ outgoing transitions from the initial state we can chose one other agent to alternate with along the path. To take care of seriality for the other agents, within each cluster also add reflexive access for any other agent, different from those two (those states all have incoming transitions and thus are reflexive so adding reflexive access for some agents doesn't change anything).
\end{itemize}

We get the following.
\begin{theorem}
    For any $n \geq 2$, $\sys$ is sound and complete wrt.\ the class of all \mc{KD45_n} models.
\end{theorem}

\section{Reflexive transitive closure}
\label{sec:reflexive}

Sometimes (e.g., in 
\cite{vanditmarsch:2003:concurrent,jamroga:2004b,vDvdHK2007}),
the \emph{reflexive}
transitive closure is used instead of transitive closure -- albeit mostly in the case of S5 knowledge when the two definitions coincide. We still give the following result for the KD45 case, with reflexive transitive closure. Recall that S4 = KT4.

\begin{theorem}
    S4 over the language $\lang$, where $C\phi$ is interpreted by taking the reflexive transitive closure, is sound and complete wrt.\ all \mc{KD45_n} models, for any $n \geq 1$.
\end{theorem}
\begin{proof}
    Proof sketch: this can be shown exactly like in the transitive case. Note that both \axiom{Cc} and \axiom{Cn} follow from reflexivity (the \axiom{T} axiom). In this case the proof can be simplified, since also the initial state is reflexive.
\end{proof}

\section{Conclusions}
\label{sec:conclusions}

Not only is common belief based on individual KD45 belief not merely KD4, but the additional properties are different for different numbers of agents.  This is obfuscated by the standard axiomatisations in terms of individual belief and induction axioms/rules. It is also similar to the case for somebody-knows \cite{sk}.

By adding the \axiom{Cc} and \axiom{Cn} axioms to KD4, we get a family of logics, and we showed that each of them is a complete characterisation of common belief for $n$ KD45 agents, in the language with only one common belief operator for the grand coalition.

The proof depends crucially on the ability to alternate agents between the clusters, which is why $n \geq 2$ is required: the corresponding result for one agent does not hold (otherwise KD45 would be equal to KD4CcC1, which it is not).

This depends on the standard interpretation in a class of models with a finite and fixed number of agents. Of course, since we (unlike in the case of one common belief operator for each group of agents) don't have agent names in the syntax, that is not necessary -- we could, alternatively, interpret the language in the broader classes of (1) all models with any finite number of agents, or (2) all models with countably infinitely many agents. We have to leave out details due to the restricted space but in both cases the resulting logic is completely axiomatised by KD4Cc (i.e., by dropping the counting axiom). 

We also get a corresponding result for the case when individual belief is K45, by adapting the proof for the KD45 case (again, we have to leave out the details): the resulting logic is K4CcCn. Other cases, like K5 and KB, are left for future work.

\bibliographystyle{plain}
\bibliography{bib}

\section*{Appendix}

\begin{definition}[Bisimulations]
Let $M^1 = (S^1, R^1, V^1)$ and $M^2 = (S^2, R^2, V^2)$ be two models. We say that $M^1$ and $M^2$ are \emph{bisimilar} (denoted $M^1 \leftrightarrows M^2$) if there is a non-empty relation $Z \subseteq S^1 \times S^2$, called a \emph{bisimulation}, such that for all $sZt$:
\begin{description}
\item[\textit{Atoms}] for all $p \in \prop$: $s \in V^1(p)$ if and only if $t \in V^2(p)$,
\item[\textit{Forth}] for all $i \in \agents$ and  $u \in S^1$ s.t. $s R^1_i u$, there is a $v \in S^2$ s.t. $t R^2_i v$ and $uZv$,
\item[\textit{Back}] for all $i \in \agents$ and $v \in S^2$ s.t $t R^2_i v$, there is a $u \in S^1$ s.t. $s R^1_i u$ and $uZv$. 
\end{description} 
We say that $M^1,s$ and $M^2,t$ are bisimilar and denote this by ${M^1,s \leftrightarrows M^2,t}$ if there is a bisimulation linking states $s$ and $t$.
\end{definition}

We have that for any modal language with a normal (relational) semantics, including $\lang$, if ${M^1,s \leftrightarrows M^2,t}$ then for any formula $\phi$, $M^1,s \models \phi$ iff $M^2,t \models \phi$.

\end{document}